\documentclass[letterpaper]{article}
\usepackage[preprint]{aaai2027}
\usepackage[hyphens]{url}
\usepackage{graphicx}
\graphicspath{{Figures/}{./figures/}{./}}
\usepackage{natbib}
\usepackage{caption}
\usepackage{booktabs}
\usepackage{amsmath}
\usepackage{amssymb}
\usepackage{amsthm}
\newtheorem{proposition}{Proposition}
\title{The Mechanism Matters: When Knowledge Graphs Help\\
Reinforcement Learning}

\author{Mohammed Sameer Syed}
\affiliations{mohammedsameersyed1@gmail.com}

\begin{document}
\maketitle

\begin{abstract}
Knowledge graphs (KGs) are widely used to inject prior knowledge into reinforcement
learning (RL), yet the literature is dominated by single-domain, positive-result method
papers, so we lack a systematic account of \emph{when} KG structure helps an agent, when it
is neutral, and when it hurts. We conduct a controlled study that independently varies the
RL task, the injection mechanism (state features, action masking, or potential-based reward
shaping), and KG quality. Using a synthetic, fully controllable KG over MiniGrid
environments, we report three findings. First, on compositional sparse-reward tasks
structured KG guidance improves sample efficiency and solve reliability (70\% to 97\% of
seeds), and a \emph{shuffle control} that permutes the KG's edges while preserving their
count collapses the benefit toward baseline (masking $p{=}0.0001$; shaping $p{=}0.006$), so
the gain is structural rather than generic regularization. Second, KG value scales with the
amount of task-relevant knowledge the graph contains. Third, and most consequential,
\emph{safety} depends on the mechanism: soft, optimality-preserving injection benefits from
correct knowledge and harmlessly ignores incorrect knowledge, whereas hard masking is
brittle, forbidding essential actions when the KG is incomplete or corrupted and making a
wrong KG worse than none. A UMLS-derived clinical case study on sepsis management under
offline RL is a careful null, underscoring that benefits require task structure the chosen
mechanism can exploit. Our results give practitioners concrete guidance on how, and how much,
to trust a KG when using it to guide RL.
\end{abstract}

\section{Introduction}
Injecting prior knowledge into reinforcement learning is a long-standing goal. Knowledge
graphs, structured collections of typed relations between entities, are an attractive source
of such priors: they are abundant, human-readable, and encode exactly the kind of relational
facts (a key opens a door, a lab value indicates an organ system) that a tabula-rasa agent
must otherwise rediscover from sparse reward. A large body of work reports that KG
augmentation helps RL agents in specific settings, spanning text-adventure games,
recommendation, navigation, and multi-hop reasoning. These are, however, almost uniformly
\emph{method papers} that demonstrate a positive effect of one KG-augmentation technique in
one domain. What the field lacks is the analysis-paper counterpart, the kind of controlled,
falsifiable ``what matters'' study that has proved valuable elsewhere in RL
\citep{andrychowicz2021matters,hessel2019inductive}. As a result, a practitioner holding a
knowledge graph and a reinforcement learning task has no principled answer to two basic
questions: should I expect this KG to help, and how should I inject it?

Answering these questions with natural knowledge graphs is difficult because a real KG
conflates many factors at once: its size, its completeness, its correctness, and the amount
of genuinely task-relevant knowledge it happens to contain are all entangled, and any single
comparison confounds them. We therefore take a controlled, synthetic-first approach. We hold
the agent and task fixed and vary, independently, (i) the \textbf{injection mechanism} (state
features, an action mask, or potential-based reward shaping); (ii) the \textbf{KG quality} (how
many task-critical relations the KG contains and whether they are correct); and (iii) the
\textbf{task structure}, ranging from environments with no exploitable relational structure to
compositional, sparse-reward tasks. Because we generate the KG, we can dial its quality
precisely and run a decisive \emph{structure-vs-memorization control}: shuffling the KG's edges
destroys its relational structure while preserving its size, so any benefit that survives
shuffling was never structural to begin with.

A concrete example makes the mechanism question vivid. Consider a gridworld in which the agent
must pick up a key, unlock a door, and reach a goal. A knowledge graph stating
\textsf{(key, opens, door)} can be handed to the agent in different ways: as an extra
observation feature, as a rule that masks out the ``toggle'' action unless the agent faces a
door, or as a shaping reward that nudges the agent toward states consistent with the KG. These
are not interchangeable. As we show, a hard mask extracts the most value when the KG is
perfect but breaks the task when the KG is wrong, whereas a soft shaping signal degrades
gracefully because it provably cannot change the optimal policy. Which mechanism to prefer is
therefore not a matter of taste but of how much one trusts the KG.

\paragraph{Contributions.}
(1)~A controlled experimental framework and a synthetic controllable KG that together isolate
\emph{when} and \emph{why} knowledge graphs help RL, released with fully reproducible,
seed-controlled code. (2)~A causal demonstration, via a shuffle control, that KG guidance
improves sample efficiency and reliability on compositional sparse-reward tasks because of
graph structure, not generic regularization. (3)~A dissociation between mechanisms: soft,
optimality-preserving injection is robust to KG error, while hard injection is brittle and can
make a wrong KG worse than no KG, together with a per-mechanism dose-response against the
amount of task-critical knowledge. (4)~Evidence that these effects generalize across six
environments and hold for a tabular learner as well as a deep on-policy one, with the caveat
that reward shaping's benefit is learner-dependent. (5)~A clinical case study on MIMIC-IV
sepsis management with a UMLS-derived KG under offline RL, testing transfer to a real,
high-stakes domain.

\section{Related Work}
\paragraph{Reinforcement learning over knowledge graphs.}
A long line of work uses RL to reason \emph{over} a KG, walking the graph to answer multi-hop
queries \citep{xiong2017deeppath,das2018minerva,lin2018multihopkg}, and more recently training
language-model agents to traverse a KG with reinforcement learning
\citep{wang2026kghopper}. This is the reverse of our
question: there the KG \emph{is} the environment and the agent's task is to navigate it,
whereas here the KG is an external prior supplied to an agent solving a separate control task.

\paragraph{Knowledge graphs as priors for RL.}
Closest to us are single-domain method papers showing that KG augmentation helps an RL agent:
KG injection in gridworlds \citep{wardenga2023injection}, commonsense-KG agents in text games
\citep{murugesan2021twc}, transfer via KG-based state representations
\citep{ammanabrolu2019transfer}, and KG-based world models of textual environments
\citep{ammanabrolu2021gata}. Each demonstrates a positive effect of one injection technique in
one domain; none performs a controlled ablation that varies the injection mechanism and the KG
quality while holding the task fixed, and recent taxonomy work notes that the space remains
under-characterized \citep{graphworldmodels2026}. Indeed \citet{wardenga2023injection}, who also
evaluate on MiniGrid, close by naming the influence of KG quality and completeness as an open
question; our dose-response and corruption experiments answer it directly. Our contribution is
precisely this controlled decomposition.

\paragraph{Reward shaping.}
Potential-based reward shaping is optimality-preserving \citep{ng1999shaping}, and recent work
derives the shaping potential from the RL problem's own transition graph
\citep{klissarov2020shaping} or analyzes when shaping improves sample complexity
\citep{gupta2022unpacking}. We differ in deriving the potential from an \emph{external}
symbolic knowledge graph and in studying its robustness when that graph is incomplete or
incorrect, which is what lets us contrast soft shaping against hard action masking.

\paragraph{The injection design space.}
Our three mechanisms are representative rather than exhaustive, and each sits within a broader
design space. On the shaping side, recent work generalizes potential-based shaping to preserve
optimality under complex, non-Markovian intrinsic-motivation rewards
\citep{forbes2024pbim} and even under action-dependent shaping \citep{forbes2025adops}, and
derives shaping heuristics from large language models rather than by hand
\citep{bhambri2024llmheuristics}; these are complementary sources of the same structural
guidance we obtain from a symbolic KG. On the masking side, learned neuro-symbolic action masks
\citep{han2026nsam} and shielding for safe RL \citep{alshiekh2018shielding} are designed
precisely to soften the brittleness we document for hard masks, trading strict enforcement for
robustness. We view our controlled contrast of soft shaping against hard masking as a backbone
onto which these richer variants can be added.

\paragraph{Analyses of what matters in RL.}
Large-scale empirical studies that isolate which design choices actually matter are a
recognized and valued genre in reinforcement learning
\citep{andrychowicz2021matters,hessel2019inductive}. We bring that controlled, falsifiable
lens to knowledge-graph injection, a setting where it has not previously been applied.

\section{Preliminaries}
\paragraph{Reinforcement learning.}
We model each task as a Markov decision process $(\mathcal{S},\mathcal{A},P,R,\gamma)$ with
states $\mathcal{S}$, actions $\mathcal{A}$, transition kernel $P(s'\mid s,a)$, reward
function $R(s,a)$, and discount $\gamma\in[0,1)$. A policy $\pi(a\mid s)$ induces a value
$V^\pi(s)=\mathbb{E}_\pi[\sum_{t\ge0}\gamma^t R(s_t,a_t)\mid s_0=s]$, and the agent seeks
$\pi$ maximizing expected return. Our environments are sparse-reward: a positive reward is
delivered only on task completion, which makes exploration hard and prior knowledge
potentially valuable.

\paragraph{Potential-based reward shaping.}
Given a bounded potential $\Phi:\mathcal{S}\to\mathbb{R}$, the shaping reward
\begin{equation}
F(s,a,s') \;=\; \gamma\,\Phi(s') - \Phi(s)
\label{eq:pbrs}
\end{equation}
added to $R$ leaves the set of optimal policies unchanged \citep{ng1999shaping}: shaping can
only accelerate or slow learning, never redirect it to a different optimum. This property is
central to our analysis, because it predicts that a shaping signal derived from an
\emph{incorrect} KG can at worst be uninformative, never actively harmful to the optimal
policy, in contrast to a hard constraint that can forbid the actions an optimal policy needs.

\paragraph{Knowledge graphs.}
A knowledge graph $G$ is a set of typed triples $(h,r,t)$ relating a head entity $h$ to a tail
entity $t$ by a relation $r$; we write $E$ for its edge (triple) set. In our controlled study
the entities are the objects, colors,
and states of the MiniGrid vocabulary, and a small set of triples such as
\textsf{(agent, can\_pick, key)}, \textsf{(key, opens, door)}, and \textsf{(door, blocks,
goal)} constitute the genuinely task-relevant, or \emph{critical}, knowledge that our
injection mechanisms consume.

\section{Method}
\subsection{Three Injection Mechanisms}
Given a KG $G$, we inject it into an otherwise-fixed agent in three ways, all behind a common
interface so that any performance difference is attributable to the KG and not to a change in
the agent.

\paragraph{State features.}
We augment the observation with a fixed-length vector of KG-credited progress features, for
example whether the agent is carrying a key, whether a door in view is open, and a
goal-proximity term. Each feature is credited only when $G$ contains the relevant relation, so
that degrading the KG removes the corresponding signal. This is the mildest injection: it adds
information to the observation but imposes no constraint and changes no reward.

\paragraph{Action mask.}
We expose a KG-derived mask over the object-interaction actions (pickup, toggle). Movement and
the terminal action are always enabled. For an interaction action $a$ with an associated
object category (pickup applies to \emph{pickable} objects; toggle applies to \emph{openable}
objects), let $K_a$ denote whether the KG contains \emph{any} relation of that category (does
the KG know what is pickable, or that doors open?), and let $\mathrm{app}_a(x)$ denote whether
the KG asserts that $a$ applies to the object $x$ currently faced. The mask disables $a$ iff
\begin{equation}
K_a \wedge \neg\,\mathrm{app}_a(x).
\label{eq:mask}
\end{equation}
The two conjuncts give the mechanism a dual character. With the category knowledge ($K_a$ true,
a complete correct KG), the rule is \emph{fail-closed per object}: $a$ is disabled at every
faced object the KG does not certify, which makes masking the most aggressive injection and
extracts the largest gains. Without it ($K_a$ false, a degraded KG), the rule \emph{fails open}
and $a$ is never disabled. The brittleness we document arises in between, where partial or
corrupted knowledge leaves $\mathrm{app}_a(x)$ wrong for the object on the optimal path,
forbidding an essential action.

\paragraph{Reward shaping.}
We assemble a potential from three binary progress indicators,
\begin{equation}
\begin{split}
\Phi(o) \;=\;\; & w_{\mathrm{key}}\,\mathbf{1}[\text{carrying key}] \\
                & + \; w_{\mathrm{door}}\,\mathbf{1}[\text{door open in view}] \\
                & + \; w_{\mathrm{goal}}\,\mathbf{1}[\text{goal in view}],
\end{split}
\label{eq:phi}
\end{equation}
where $\mathbf{1}[\cdot]$ is the indicator function and each indicator is credited only if the
corresponding KG relation is present, namely
\textsf{(agent, can\_pick, key)}, \textsf{(key, opens, door)}, and a \textsf{door}-to-\textsf{goal}
path respectively, and we use $w_{\mathrm{key}}{=}1$, $w_{\mathrm{door}}{=}2$,
$w_{\mathrm{goal}}{=}1$ with a global shaping scale of $0.5$. Crucially, all three indicators
are computed from the agent's \emph{partial egocentric observation} $o$; the goal term is a
local visibility indicator, not a privileged global-map distance, so no mechanism uses
information the agent could not itself observe (pseudocode and no-leak argument in
Appendix~C). We then apply Equation~\eqref{eq:pbrs}, force
$\Phi$ to zero at episode end (including time-limit truncations), and always \emph{evaluate on
the true task reward} $R$ while training on $R+F$.

Two caveats follow. Because $\Phi$ is a function of the observation rather than the underlying
Markov state, Proposition~1's invariance holds over the observation process rather than the
latent MDP, so under partial observability an observation-conditioned potential can in principle
perturb the optimal policy; and treating a truncation as terminal can introduce finite-horizon
bias. In both cases we evaluate on the true unshaped return, so any deviation manifests as a
change in \emph{learning dynamics} rather than as a bias in reported performance, and we find no
empirical sign of a distorted terminal policy (Appendix~C). Our robustness claims for
shaping are therefore empirical with respect to these two subtleties.

\subsection{Optimality Guarantees}
The soft-versus-hard contrast has a simple formal basis, which we state to make the
mechanism-dependent safety explicit.

\begin{proposition}[Soft injection preserves optimality]
For any potential $\Phi:\mathcal{S}\to\mathbb{R}$, including one derived from an arbitrary and
possibly incorrect KG, the shaped MDP with reward $R+F$ and $F(s,a,s')=\gamma\Phi(s')-\Phi(s)$
has exactly the same set of optimal policies as the original MDP.
\end{proposition}
\noindent This is the potential-based shaping invariance theorem \citep{ng1999shaping}. Its
consequence for us is that KG-derived shaping cannot make the optimal policy worse: an incorrect
KG can slow learning but never redirect it to a suboptimal optimum.

\begin{proposition}[Hard fail-open masking can destroy optimality]
Let $m$ be a fail-open action mask that disables action $a$ in state $s$ only when the KG
entails that $a$ is irrelevant in $s$. Suppose there is a state $s^\star$ visited by every
optimal policy at which the optimal action $a^\star$ is disabled by $m$. Then no optimal policy
of the original MDP is feasible under $m$, and the optimal value of the masked MDP is strictly
less than that of the original MDP.
\end{proposition}
\noindent \emph{Proof sketch.} Any policy feasible under $m$ never takes $a^\star$ at
$s^\star$; since every optimal policy must, the feasible set excludes them all and the best
feasible policy has strictly lower value. The precondition holds exactly when the KG lacks or
corrupts the critical relation that would mark $a^\star$ relevant. Soft injection's worst-case
regret from an incorrect KG is therefore a slowdown, whereas hard injection's can be as large as
the value gap to the best mask-feasible policy; Section~5 measures how often this is triggered
in practice.

\subsection{A Synthetic Controllable Knowledge Graph}
We generate the KG over the MiniGrid entity vocabulary with four independent knobs.
\emph{Completeness} $c\in[0,1]$ is the fraction of ground-truth edges retained.
\emph{Noise} $\eta$ is the fraction of edges replaced by false ones.
\emph{Size} adds distractor entities and irrelevant edges.
A \emph{shuffle} flag rewires edge endpoints, preserving the edge count while destroying
structure, and is the basis of our memorization control. Our default shuffle permutes heads and
tails independently, which also holds each entity's in- and out-degree fixed and removes only
the head-to-tail pairing; as a stronger control we additionally implement an explicit
degree-preserving edge-swap shuffle and confirm the benefit collapses under both (Section~5), so
the effect is not a degree artifact. Both preserve relation type and direction and are seeded
per run (Appendix~D). To obtain a clean
dose-response we additionally identify the small set of \emph{critical triples} that the
mechanisms actually consume, and expose two further knobs: \texttt{keep\_critical}
$k\in\{0,1,2,3\}$, which retains exactly $k$ of the critical relations while holding the rest
of the KG fixed, and \texttt{corrupt\_critical}, which replaces the retained critical
relations with wrong tails. Together these separate two distinct notions of KG quality that a
global completeness knob would confound: how much task-relevant knowledge the KG contains, and
whether that knowledge is correct.

\subsection{Study Design}
We run three controlled studies and one case study. Study~A, structure-vs-memorization,
compares each mechanism under a correct KG, a shuffled KG, and no KG. Study~B, the
critical-relation dose-response, varies $k$ and the correct/corrupt flag. Study~C, external
validity, sweeps six MiniGrid environments and two learning algorithms. The case study applies
the same mechanisms to offline RL on a real clinical dataset. Because the hardest tasks are
bimodal across seeds (a run either solves the task or does not), we report \emph{solve rate},
the fraction of seeds reaching the goal, alongside the area under the learning curve (AUC) as a
variance-robust separator.

\section{Experimental Setup}
\paragraph{Environments.}
We use six MiniGrid \citep{chevalierboisvert2023minigrid} tasks spanning a spectrum of relational structure.
\texttt{Empty-5x5} and \texttt{LavaGapS5} contain no key/door structure and serve as
negative controls where the KG should be neutral.
\texttt{DoorKey-5x5}, \texttt{DoorKey-6x6}, \texttt{Unlock}, and \texttt{KeyCorridorS3R1}
require the compositional sequence of retrieving a key, opening a door, and reaching a goal or
object, and are where relational prior knowledge is expected to help. All use the default
partial-observation image encoding.

\paragraph{Agents.}
Our primary learner is Proximal Policy Optimization (PPO) \citep{schulman2017ppo} with a compact multilayer-perceptron
actor-critic. For the algorithm-generality study we add a tabular Q-learning agent that hashes
the observation to a discrete state, a classical value-based learner with no function
approximation. Unless stated otherwise we train for $1.2{\times}10^5$ environment steps and
report over ten seeds. Full hyperparameters are listed in Appendix~A.

\paragraph{Metrics and reproducibility.}
For each run we log the true-reward learning curve, summarize it by its AUC and by the final
solve rate, and aggregate across seeds. For AUC we report means with 95\% Student-$t$
confidence intervals and, for the key contrasts, Welch tests and Cohen's $d$ effect sizes; for
solve rates, which are proportions, we report Wilson score intervals rather than normal
intervals. We do not apply a multiple-comparison correction and therefore treat any marginal
results ($p \gtrsim 0.05$) as directional. Every run is fully seeded across Python, NumPy, and
PyTorch, and the tabular discretization uses a content hash that is stable across process
launches so that results reproduce exactly. Experiment configurations are declarative and the
grid runner is resumable. \texttt{DoorKey-5x5} recurs across several studies that use different
seed budgets (thirty for the structure test, fifteen for the dose-response, ten elsewhere), so
its absolute AUC differs slightly between tables; every claim we make is a contrast computed
\emph{within} a single study, never across them.

\section{Results}
\subsection{Structured Guidance Helps, and the Benefit Is Structural}
On \texttt{DoorKey-5x5} (thirty seeds, PPO), reward shaping and action masking speed up
learning sharply, with AUC rising from 0.39 (control) and 0.43 (naive state features) to 0.65
(shape) and 0.69 (mask). The shuffle control is decisive
(Figure~\ref{fig:struct}). Permuting the KG's edges collapses the AUC benefit back toward the
control level: masking falls from 0.69 to 0.34 (Welch $p{=}0.0001$, Cohen's $d{=}1.08$) and
shaping from 0.65 to 0.45 (Welch $p{=}0.006$, $d{=}0.75$), so the benefit derives from graph
structure rather than from generic reward densification or feature regularization. Solve
reliability tells a consistent story:
shaping reaches 29/30 seeds (Wilson 95\% CI $[0.83, 0.99]$) and masking 26/30 ($[0.70, 0.95]$)
versus 21/30 for the control ($[0.52, 0.83]$); because DoorKey-5x5 is eventually solved by most
seeds, the mechanisms' advantage shows up primarily as faster learning (AUC) rather than as a
higher asymptotic solve count. Naive state features neither help nor respond to
shuffling, consistent with their being a re-encoding of information the agent already
observes. The collapse is not an artifact of edge count or degree: for the tabular learner,
masking reaches AUC $0.38$ with the correct KG but falls to $0.29$ under endpoint permutation
and to $0.21$ under an explicit degree-preserving edge-swap shuffle, both at or below the
$0.26$ no-KG control (degree-preserving versus correct: $d{=}1.52$, $p{=}0.005$). What matters
is therefore the head-to-tail pairing, not the edge count or the degree sequence.

\begin{figure}[t]
\centering
\includegraphics[width=\columnwidth]{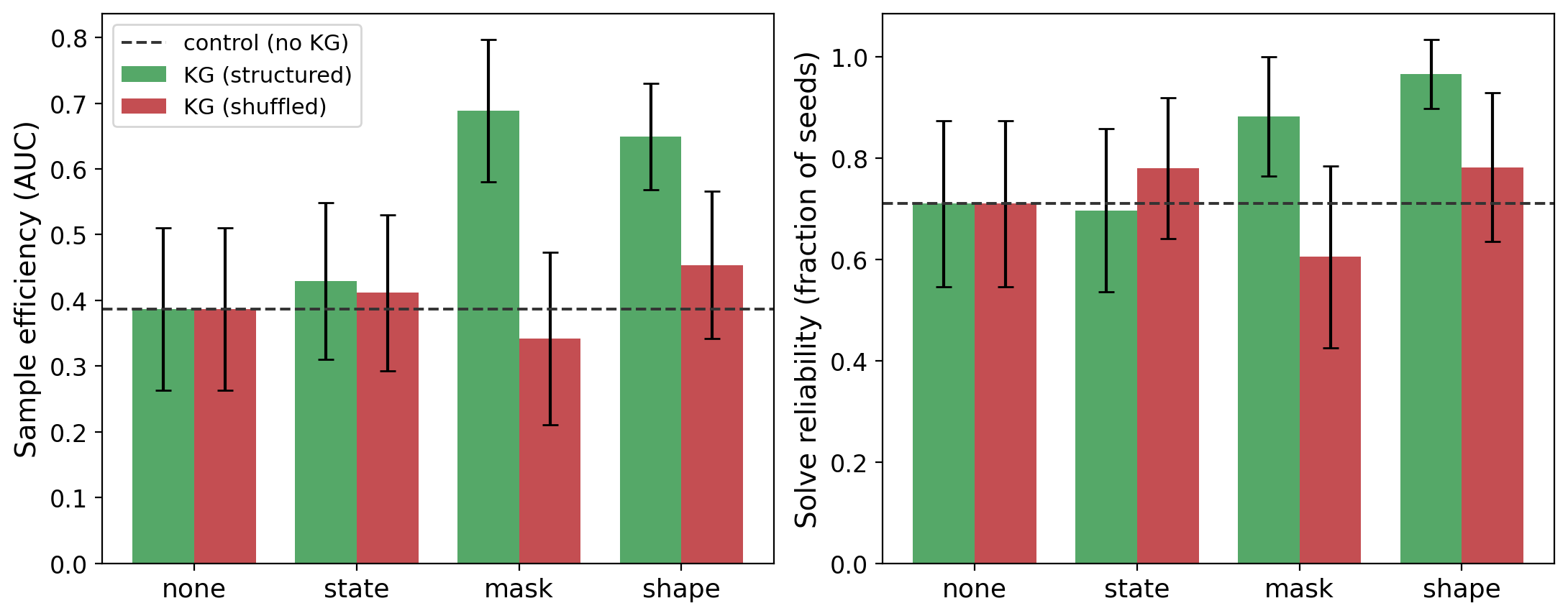}
\caption{Structure-vs-memorization (DoorKey-5x5, thirty seeds, PPO). For \texttt{mask} and
\texttt{shape}, structured-KG bars exceed the no-KG control (dashed line), and shuffling the
KG collapses the benefit toward control, significantly for both \texttt{mask} ($p{=}0.0001$) and
\texttt{shape} ($p{=}0.006$). \texttt{state} shows no effect. Error bars are
95\% $t$-confidence intervals over thirty seeds.}
\label{fig:struct}
\end{figure}

\subsection{KG Value Scales with Knowledge; the Mechanism Decides Safety}
Varying the number of task-critical relations from zero to three yields a three-way
dissociation (Figure~\ref{fig:dose}; fifteen seeds). Soft \textbf{shape} injection shows a
graded dose-response, with AUC rising as more correct critical relations are present (0.36,
0.53, 0.52, 0.66), and when the same relations are corrupted it becomes statistically
indistinguishable from the no-KG control (0.36, 0.36, 0.36, 0.42 against a 0.39 control; pooled
$d{=}{-}0.04$, $p{=}0.87$). In other words shaping benefits from
correct knowledge and harmlessly ignores incorrect knowledge, exactly as the
optimality-preservation property predicts. Hard \textbf{mask} injection behaves oppositely, and
fails in two distinct ways. Given correct but \emph{incomplete} knowledge it degrades steeply
(0.00, 0.27, 0.34, 0.75), attaining its peak only when all three critical relations are present.
Given \emph{corrupted} knowledge it collapses to exactly 0.00 at every dose, meaning the task
becomes unsolvable. A hard constraint built from an imperfect KG
forbids the very actions an optimal policy requires, so a wrong KG is strictly worse than no
KG. \textbf{State} injection shows no dose-response, fluctuating non-monotonically within a
narrow band.

\begin{figure}[t]
\centering
\includegraphics[width=\columnwidth]{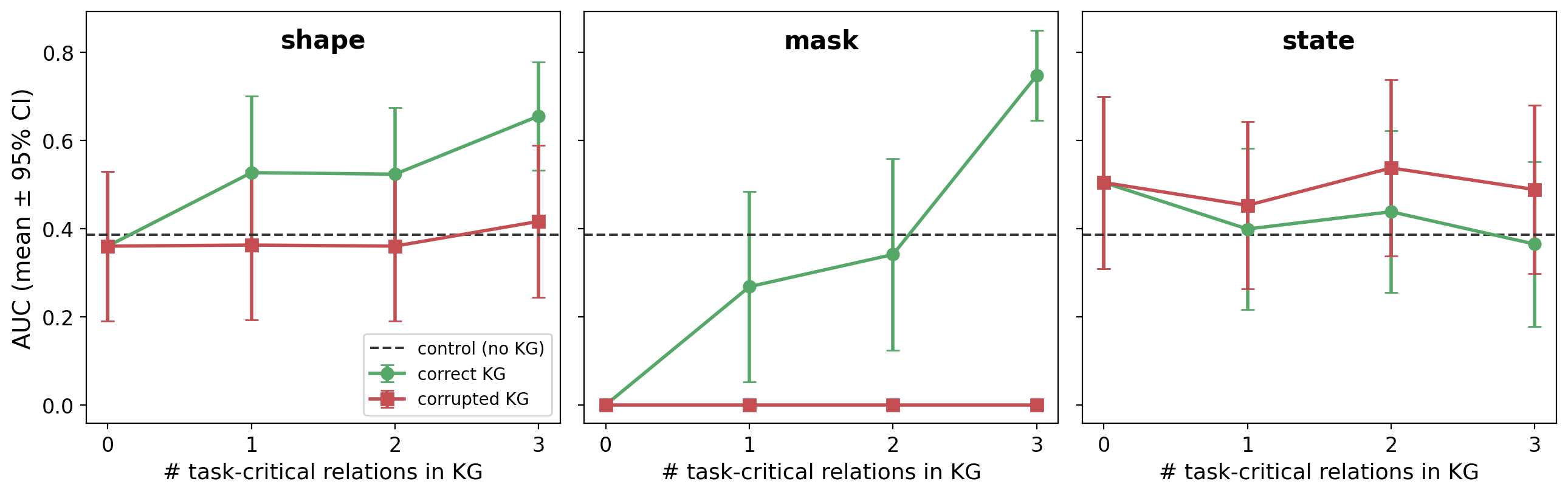}
\caption{Dose-response over the number of task-critical relations in the KG (DoorKey-5x5,
PPO, fifteen seeds). Soft injection (\texttt{shape}) rises with correct knowledge and falls back
to the control level under corruption; hard injection (\texttt{mask}) needs complete and
correct knowledge and breaks otherwise; naive \texttt{state} features do not respond.}
\label{fig:dose}
\end{figure}

To make the masking failure mode explicit, Table~\ref{tab:blockrate} measures how often the mask
forbids \emph{essential} actions over random rollouts. A complete and correct KG prunes about a
quarter of the action space yet never blocks one; as critical relations are removed or
corrupted, it blocks essential actions most of the time, which is exactly what renders the task
unsolvable. Notably, a degraded KG masks \emph{fewer} actions overall but blocks the
\emph{wrong} ones, so the harm comes from misdirected pruning rather than from more pruning.

\begin{table}[t]
\centering
\small
\begin{tabular}{lcc}
\toprule
KG condition & actions & essential \\
             & masked  & blocked   \\
\midrule
complete \& correct ($k{=}3$) & 26.2\% & \textbf{0.0\%} \\
partial ($k{=}2$)             & 21.9\% & 28.6\% \\
partial ($k{=}1$)             & 17.8\% & 53.3\% \\
none ($k{=}0$)                & 14.3\% & \textbf{81.6\%} \\
corrupted ($k{=}3$)           & 13.4\% & \textbf{81.6\%} \\
\bottomrule
\end{tabular}
\caption{Why hard masking is brittle (DoorKey-5x5, seeded random rollouts: 50 seeds
$\times$ 400 steps, 2595 essential opportunities per condition). ``Actions masked'' is the
mean per-step fraction of the action space disabled; ``essential blocked'' is the fraction of
key-pickup and door-toggle opportunities the mask forbids (key pickup when facing the key, door
toggle when facing the door while carrying it).}
\label{tab:blockrate}
\end{table}

\paragraph{Softening the mask trades peak gain for robustness.}
The brittleness above is specific to \emph{hard} masking. A natural mitigation, suggested by
work on graded and shielded constraints, is a \emph{soft} mask that subtracts a penalty from the
logits of KG-irrelevant actions, discouraging rather than forbidding them, so an essential
action remains reachable if the policy insists (Table~\ref{tab:soft}). Soft masking is more
reliable under \emph{partial} knowledge, raising the solve rate from $0.30$ to $0.47$ at one
critical relation and from $0.45$ to $0.67$ at two, precisely because it never renders an
essential action unavailable. The cost is a lower ceiling on a complete correct KG (AUC $0.52$
versus $0.77$), and at our penalty it does not rescue the corrupted case ($0.10$ versus $0.00$).
The penalty is therefore a tunable dial between the hard mask's high-gain, high-risk profile and
a safer, lower-ceiling one.

\begin{table}[t]
\centering
\small
\begin{tabular}{lccccc}
\toprule
& $k{=}0$ & $k{=}1$ & $k{=}2$ & $k{=}3$ & $k{=}3$ corrupt \\
\midrule
hard mask & 0.00 & 0.30 & 0.45 & \textbf{1.00} & 0.00 \\
soft mask & 0.10 & \textbf{0.47} & \textbf{0.67} & 0.89 & 0.10 \\
\bottomrule
\end{tabular}
\caption{Hard versus soft masking (solve rate, DoorKey-5x5, ten seeds, PPO) across the number
of task-critical relations $k$. Soft masking is more reliable under partial knowledge
($k{=}1,2$) but has a lower ceiling at a complete correct KG ($k{=}3$); it exposes a tunable
peak-gain-versus-robustness trade-off.}
\label{tab:soft}
\end{table}

\subsection{The Pattern Generalizes Across Environments}
Table~\ref{tab:multienv} reports AUC across the six-task suite (ten seeds, PPO), and
Figure~\ref{fig:multienv} shows the corresponding solve rates. The result
splits cleanly by task structure. On the two tasks with no exploitable relational structure,
\texttt{Empty-5x5} and \texttt{LavaGapS5}, every injection mode is close to the control (all
AUC between 0.75 and 0.91), so the KG is neutral, as it should be. On the four key/door tasks
the KG helps, and it helps most where the baseline is weakest: on \texttt{DoorKey-6x6},
\texttt{Unlock}, and \texttt{KeyCorridorS3R1} the no-KG control essentially never learns
(AUC $\le 0.02$, solve rate near zero) while masking reaches 0.21 to 0.37 and shaping 0.14 to
0.18. Across all six tasks, action masking is the strongest and most reliable injection,
shaping is second, and naive state features track the control. This supports the central
empirical claim: KGs help RL in proportion to the task-relevant structure the task contains,
and are neutral when there is none to exploit.

\begin{table}[t]
\centering
\small
\begin{tabular}{lcccc}
\toprule
Environment & none & state & mask & shape \\
\midrule
Empty-5x5        & 0.84 & 0.91 & 0.87 & \textbf{0.91} \\
LavaGapS5        & 0.83 & 0.75 & \textbf{0.89} & 0.79 \\
\midrule
DoorKey-5x5      & 0.37 & 0.37 & \textbf{0.77} & 0.66 \\
DoorKey-6x6      & 0.01 & 0.06 & \textbf{0.37} & 0.18 \\
Unlock           & 0.02 & 0.07 & \textbf{0.36} & 0.14 \\
KeyCorridorS3R1  & 0.00 & 0.00 & \textbf{0.21} & 0.18 \\
\bottomrule
\end{tabular}
\caption{Sample efficiency (AUC of the learning curve, mean over ten seeds, PPO) across the
environment suite. Top block: no relational structure, where the KG is neutral. Bottom block:
key/door composition, where the KG helps and most where the baseline fails. Best injection per
row in bold.}
\label{tab:multienv}
\end{table}

\begin{figure}[t]
\centering
\includegraphics[width=0.85\columnwidth]{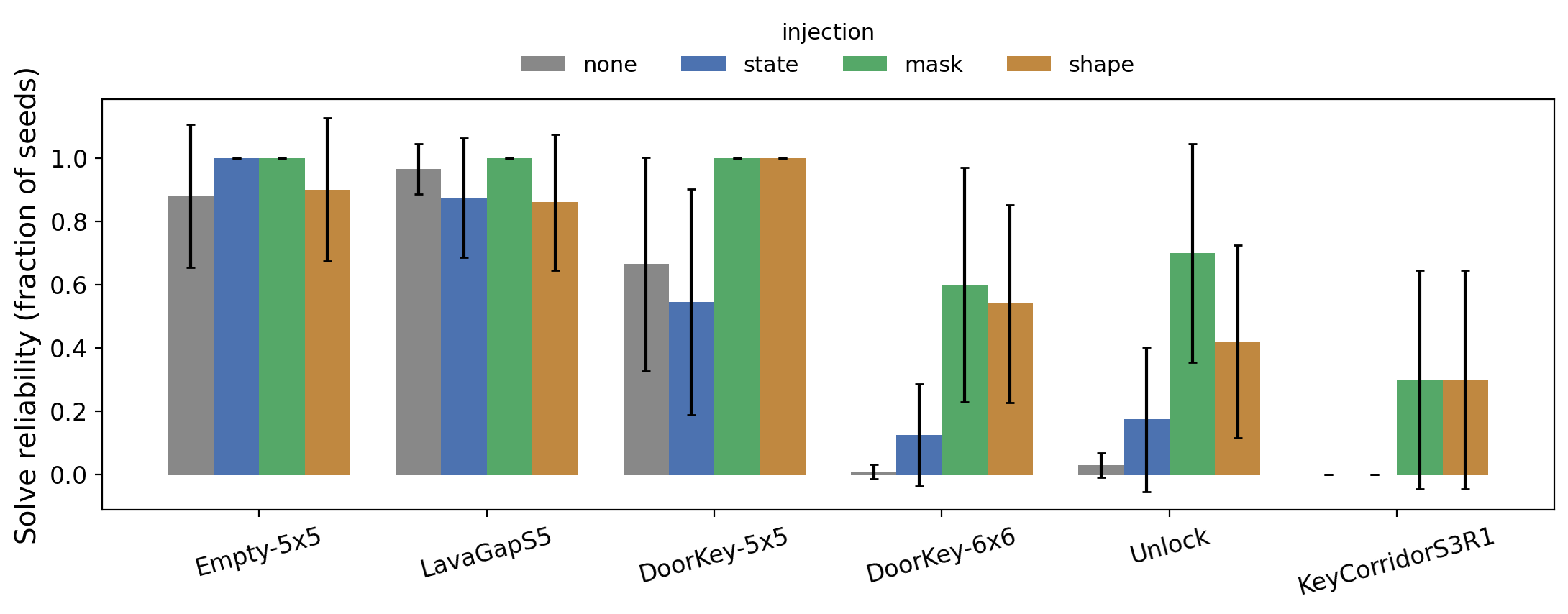}
\caption{Solve reliability across the six-environment suite (fraction of ten seeds reaching
the goal, PPO). Knowledge-graph injection (\texttt{mask}, \texttt{shape}) raises reliability on
the key/door tasks, most where the no-KG control fails, and is neutral on the unstructured
tasks (\texttt{Empty}, \texttt{LavaGap}); naive \texttt{state} features track the control.}
\label{fig:multienv}
\end{figure}

\subsection{Algorithm Generality}
We repeat the structure-vs-memorization study with the tabular Q-learning agent as a
contrasting learner. Action masking helps both learners and is structural in both: for the
tabular agent, masking raises AUC from 0.26 (control) to 0.38 and the solve rate from 0.37 to
0.51, and shuffling the KG collapses the benefit back to 0.29. Reward shaping, by contrast, is
learner-dependent. At the shaping scale tuned for PPO ($0.5$), the same shaping signal that
helps PPO destabilizes tabular Q-learning (AUC 0.13 and solve rate 0.03, well below the 0.26
control), because the dense, view-conditioned shaping reward is large relative to the sparse
task reward and the tabular updates cannot absorb the resulting variance, whereas PPO's
batched, advantage-normalized updates can.

A scale ablation (Table~\ref{tab:scale}) shows that the magnitude matters but does not by itself
explain the failure. Lowering the scale from the PPO-tuned $0.5$ to $0.1$ restores AUC to about
the no-KG control level ($0.32$ versus $0.30$), so most of the degradation at larger scales is a
scale effect. Solve reliability, however, does not recover: it is $0.00$ at every non-zero scale
we tried, against $0.53$ for the control. We therefore find no shaping scale at which this
tabular learner benefits, and we report the learner-dependence of shaping as a genuine
limitation rather than a tuning artifact.

\begin{table}[t]
\centering
\small
\begin{tabular}{lcc}
\toprule
Shaping scale & AUC & solve rate \\
\midrule
0.0 (control)   & 0.30 & \textbf{0.53} \\
0.1             & \textbf{0.32} & 0.00 \\
0.25            & 0.14 & 0.01 \\
0.5 (PPO-tuned) & 0.14 & 0.02 \\
1.0             & 0.13 & 0.02 \\
2.0             & 0.13 & 0.02 \\
\bottomrule
\end{tabular}
\caption{Shaping-scale ablation for the tabular learner (DoorKey-5x5, five seeds). Reducing the
scale recovers the AUC lost at larger scales but does not restore solving: the solve rate is
zero at every non-zero scale. This is a separate five-seed sweep, so its control row is not
directly comparable to the ten-seed tabular control ($0.26$) quoted in the text; read the scale
comparison within the table.}
\label{tab:scale}
\end{table}

The message is therefore that shaping's benefit is learner-dependent in a way masking's is not:
masking helps both learners and is structural in both, whereas shaping helps PPO and, at every
scale we tried, does not help the tabular learner. We also attempted a deep value-based learner
(DQN) as a third algorithm but found
vanilla DQN too unstable on these sparse-reward tasks to serve as an informative comparison,
as it fails even \texttt{Empty-5x5} on a majority of seeds; we report this as a limitation
rather than as a KG result.

\subsection{Does a Structure-Aware Encoder Rescue State Injection?}
A natural objection is that our state features are too naive: because they re-encode the
observation, of course they do not help. We test progressively stronger representations of the
KG through the state channel: a \emph{fixed} graph convolutional encoder over the KG
(Appendix~A) whose node embeddings depend on the graph's multi-hop connectivity, and a
\emph{jointly-trained} GNN whose node features and weights are learned end-to-end with the PPO
policy. On \texttt{DoorKey-5x5} the two encoders form a clean monotone ordering above the
control and naive features: AUC $0.37$ (control) $\approx 0.37$ (state) $< 0.48$ (fixed GNN)
$< 0.585$ (jointly-trained GNN), with solve rate rising from $0.67$ to $0.80$; the
jointly-trained encoder is the strongest state-channel variant, and its benefit appears
structural, since shuffling the KG reduces it from $0.585$ to $0.39$, though at ten seeds this
shuffle contrast is directional rather than significant ($d{\approx}0.57$, $p{\approx}0.22$).
So end-to-end representation learning does appear to extract structure-dependent value from the
state channel that naive features cannot (Cohen's $d \approx 0.70$ versus control).

Yet three caveats keep this from overturning the mechanism story. The improvement over control
is only medium-sized and not significant at ten seeds ($p \approx 0.14$); it does not transfer
to the harder \texttt{DoorKey-6x6} (AUC $0.01$, like the control); and even the jointly-trained
encoder ($0.585$) does not reach reward shaping ($0.66$) or action masking ($0.77$) on the same
task. Representation learning therefore \emph{narrows but does not close} the gap between the
state channel and the shaping or masking channels. The state channel adds observation features
without shaping the objective or constraining the action space, and it is this mechanism-level
limitation, rather than the sophistication of the representation, that bounds its benefit.

\section{Clinical Case Study: Sepsis Management}
We test whether the general-domain dissociation transfers to a real, high-stakes domain:
sequential treatment decisions for sepsis. From MIMIC-IV \citep{johnson2023mimiciv} we build a cohort of 16{,}912 ICU
stays (11{,}996 patients, 29.6\% in-hospital mortality) selected by sepsis diagnosis codes. We
discretize each stay into 4-hour windows with a 20-dimensional state of demographics, vitals,
and laboratory values, a 25-way action over binned intravenous-fluid volume and vasopressor
rate, and a
terminal reward of $+15$ for survival and $-15$ for death. Because exploration on patients is
infeasible, we train a discrete Conservative Q-Learning (CQL) policy \citep{kumar2020cql} offline, and we inject a
UMLS-derived KG \citep{bodenreider2004umls} (59 clinical concepts and 736 relations, anchored to the state features by 63
concept-to-feature mappings) through the state and shape
mechanisms; we do not report the mask mechanism here because a clinically valid
action-indication mask is a substantial undertaking of its own and is left to future work.
Policies are evaluated by Fitted-Q Evaluation (FQE) \citep{le2019batch} of the initial-state value on the true
mortality reward, so that, as in the general-domain study, a shaped-reward policy is scored on
true outcomes rather than on its shaping bonus.

\begin{table}[t]
\centering
\small
\begin{tabular}{lc}
\toprule
Injection & FQE value (5 seeds, 95\% $t$-CI) \\
\midrule
none (control) & $8.83 \pm 0.29$ \\
state          & $8.79 \pm 0.27$ \\
shape          & $8.87 \pm 0.28$ \\
\bottomrule
\end{tabular}
\caption{Off-policy value (FQE) of the learned sepsis policy under each KG-injection
mechanism (MIMIC-IV, five CQL+FQE seeds). With training-seed variance included the three
mechanisms are statistically indistinguishable; the differences are far smaller than the
confidence intervals.}
\label{tab:clinical}
\end{table}

The clinical result (Table~\ref{tab:clinical}) is a careful null: across five training seeds the
three mechanisms' FQE values lie within $0.08$ of one another while the 95\% confidence
intervals span roughly $\pm 0.28$. A single-seed run had suggested that state features lower the
value and shaping raises it, but those apparent effects are smaller than the training-seed
variance and do not survive replication, which is precisely why multi-seed evaluation matters.
We read this as consistent with, rather than contrary to, the general-domain thesis: KG
injection helps when a task exposes relational structure the chosen mechanism can exploit, and a
coarse offline discretization of sepsis management evaluated by FQE may simply not surface such
structure at this scale. A coverage analysis reinforces this: of the five task-critical clinical
concepts our shaping and state mechanisms rely on, only three are connected in the UMLS-derived
KG (tissue hypoxia and organ dysfunction have no retained relations). At this roughly $60\%$
coverage, our general-domain dose-response (Figure~\ref{fig:dose}) would itself predict a weak
or absent effect, so the null plausibly reflects insufficient KG coverage rather than a failure
of KG injection in principle. We therefore present the case study strictly as a methodological
null and not as evidence for or against clinical benefit, in keeping with the well-documented
validity concerns of offline RL in healthcare \citep{gottesman2019guidelines}.

\section{Discussion}
The central practical takeaway is that ``does the KG help?'' is the wrong question; the useful
ones are ``how much task-relevant knowledge does the KG contain?'' and ``how am I injecting
it?'' Our results support three pieces of guidance. First, expect a KG to help in proportion to
the task's relational structure and to the fraction of task-critical relations it actually
encodes; a KG that omits the relations an agent needs provides no benefit however large it is.
Second, prefer soft, optimality-preserving injection (reward shaping) when the KG may be
incomplete or noisy, since it cannot change the optimal policy and a wrong KG is at worst
neutral for a sufficiently robust learner. Third, treat hard constraints (action masking) as
high-reward but high-risk: masking is our single best injection across environments and extracts
the most from a complete correct KG, but an incomplete or corrupted mask can forbid essential
actions and make a wrong KG strictly worse than no KG. Both axes interact with the learner:
masking's benefit transfers across a deep on-policy learner and a tabular one, whereas shaping's
depends on the learner absorbing a dense auxiliary signal.

\paragraph{Limitations.}
Our controlled study uses gridworld tasks; the clinical case study probes transfer to a real
domain, but offline-RL evaluation in healthcare is itself contested and we frame it as a
methodological case study rather than a clinical recommendation. The mask brittleness we report
reflects a fail-open rule that infers irrelevance from partial knowledge, and other masking
semantics may be less brittle; characterizing that design space is future work. Our hardest
tasks are bimodal across seeds, which widens confidence intervals; we mitigate this with
multi-seed runs (up to thirty for the structure test) and by reporting solve rate alongside AUC.
Our state features are deterministic functions of the observation the agent already receives, so
they re-encode rather than add information; ``state does not help'' should therefore be read as
``this simple re-encoding does not help'' rather than as a claim that KG-derived state
augmentation fails in general, and for a tabular learner such features are inert by
construction. We report several two-sample tests without multiple-comparison correction: the
main structure-vs-memorization effects (mask, $p{=}0.0001$; shape, $p{=}0.006$) are below any
reasonable corrected threshold, while smaller secondary contrasts should be read as directional.
Two negative results are scoped to what we varied: no shaping \emph{scale} helps the tabular
learner, though we did not try eligibility traces or annealed soft-mask penalties, and our
corruption perturbs relations independently rather than in correlated, schema-level ways.
Finally, we study three representative mechanisms and two learners; both a
fixed and a jointly-trained GNN-over-KG encoder improve the state channel without making it
competitive with shaping or masking (Section~5), and extending to further algorithm families and
richer mechanisms remains a natural next step.

\section{Conclusion}
We give a controlled account of when knowledge graphs help reinforcement learning, separating
task structure, injection mechanism, and KG quality. KGs help in proportion to the task-relevant
knowledge they carry, the benefit is causally structural rather than a regularization artifact,
and the mechanism decides whether an imperfect KG is harmless or harmful.

%
\bibliography{references}

\clearpage
\appendix
\noindent This appendix provides supplementary material for the main text. Section references
of the form ``Appendix~A'' above refer to the correspondingly lettered sections below.

\section{Agent Hyperparameters}
PPO uses a two-layer multilayer perceptron (128 hidden units, tanh activations) for both actor
and critic, a rollout length of 128 steps, four optimization epochs of four minibatches per
update, a learning rate of $2.5{\times}10^{-4}$, discount $\gamma{=}0.99$, GAE parameter
$\lambda{=}0.95$, clipping coefficient 0.2, entropy coefficient 0.01, value coefficient 0.5,
and gradient-norm clipping at 0.5. The tabular agent uses learning rate 0.1, discount 0.99,
and an $\epsilon$-greedy schedule annealing from 1.0 to 0.05 over the first 60\% of training.
The clinical CQL and FQE models use the library defaults with discount 0.99. Reward shaping
uses a scale of 0.5 with potential weights of 1, 2, and 1 for the carry-key, door-open, and
goal-proximity terms respectively. Both GNN-over-KG encoders use two graph-convolutional
layers with an 8-dimensional embedding over a symmetric normalized adjacency with self-loops.
The fixed encoder draws random node features and layer weights once and holds them constant;
the jointly-trained encoder makes the node features and layer weights learnable parameters of
the actor-critic, updated by the PPO objective, and pools the embeddings of the currently
visible entities into the policy input.

\paragraph{Computing infrastructure.}
All gridworld experiments run on CPU on an Apple M4 Max with 36\,GB of unified memory under
macOS 26.5.2; no GPU is required, and grid cells are parallelized across cores with each worker
pinned to a single thread. The gridworld stack is Python 3.14 with PyTorch 2.13.0, Gymnasium 1.0.0,
MiniGrid 3.1.0, NumPy 2.5.1, SciPy 1.18.0, NetworkX 3.6.1, and Matplotlib 3.11.0. The offline
clinical study runs in a separate Python 3.11 environment with d3rlpy 2.8.1 and the same
PyTorch version, because d3rlpy does not yet ship wheels for the newer interpreter.

\section{Clinical Cohort and Features}
The cohort comprises adult ICU stays whose hospital admission carries a sepsis ICD-9 or ICD-10
code. The state vector concatenates demographics (age, sex) with per-window means of heart
rate, systolic, diastolic and mean arterial pressure, respiratory rate, temperature, oxygen
saturation, and Glasgow Coma Scale, together with lactate, creatinine, blood urea nitrogen,
white-cell count, platelets, bilirubin, sodium, potassium, pH, and glucose; missing values are
forward-filled within a stay and then median-imputed. Actions bin total intravenous-fluid
volume and maximum vasopressor rate per window into a 5-by-5 grid. The UMLS KG is restricted to
relations among the anchored clinical concepts and connects, for example, acute kidney injury
to creatinine and blood urea nitrogen, tissue hypoxia to lactate and oxygen saturation, and
hypotension to the blood-pressure features.

\paragraph{Per-concept KG coverage.}
Both the \texttt{state} and \texttt{shape} mechanisms consume the KG through five task-critical
clinical concepts, each mapped to a group of state features; a concept contributes signal only
if the retained KG actually connects it. Table~\ref{tab:coverage} reports, for each concept, the
features it pools and its degree in the extracted 59-node, 736-edge graph. Three of the five
concepts are connected and two are isolated, so the mechanisms operate at roughly $60\%$
coverage of the knowledge they were designed to use. Under the general-domain dose-response
(main paper, Figure~2) that level of coverage predicts a weak or absent effect, which is the
basis for reading the clinical result as coverage-limited rather than as evidence against KG
injection.

\begin{table}[h]
\centering
\small
\begin{tabular}{llc}
\toprule
Clinical concept & features pooled & degree \\
\midrule
Acute kidney injury & creatinine, BUN           & 62 \\
Hypotension         & SBP, DBP, MAP             & 128 \\
Sepsis              & WBC, temp., HR, RR        & 100 \\
Tissue hypoxia      & lactate, SpO$_2$          & \textbf{0} \\
Organ dysfunction   & GCS, bilirubin, platelets & \textbf{0} \\
\bottomrule
\end{tabular}
\caption{Coverage of the five task-critical clinical concepts in the UMLS-derived KG. A degree
of zero means the concept survives anchoring but retains no relations, so its channel is
identically zero and contributes no signal.}
\label{tab:coverage}
\end{table}

\section{Shaping Potential and Goal Proximity}
The potential $\Phi$ is a sum of three non-negative terms, each gated by a KG relation so that
credit is assigned only when the graph actually encodes the corresponding prior. Crucially,
every quantity $\Phi$ reads is a function of the agent's \emph{egocentric partial observation}
(the native MiniGrid $7\times7\times3$ view) or of the abstract KG; $\Phi$ never accesses the
global grid, the agent's absolute coordinates, or the goal's location. The pseudocode is:

\begin{small}
\begin{verbatim}
Phi(obs, KG):  # obs = egocentric 7x7 view
  phi = 0
  if carrying_key() and
     KG.has(agent, can_pick, key):
    phi += w_key
  if door_open_in_view(obs) and
     KG.has(key, opens, door):
    phi += w_door
  if goal_in_view(obs) and
     KG.path_len(door, goal) is not None:
    phi += w_goal
  return phi
\end{verbatim}
\end{small}

\noindent The goal-proximity indicator \texttt{goal\_in\_view} is a Boolean: true only when a
goal object already falls inside the agent's local field of view, information it receives in its
observation regardless of shaping. The KG term \texttt{path\_len(door,goal)} is a hop count over
the \emph{abstract} graph, not over grid cells, so it conveys ``the KG believes a door leads to
a goal,'' not where the goal is. Consequently $\Phi$ cannot leak privileged map knowledge: an
agent with a correct KG but no view of a goal, or a view of a goal but a KG lacking the
door--goal link, receives no goal-term credit. At episode termination or truncation we set
$\Phi=0$, and we always evaluate on the unshaped environment reward.

\paragraph{Does observation-conditioned shaping distort the terminal policy?}
Because $\Phi$ reads the observation rather than the latent Markov state, the invariance of
Proposition~1 of the main paper holds over the observation process, and in principle the optimal
policy could be perturbed. We find no empirical sign of this. If the potential were pushing the agent toward a
suboptimal fixed point, the clearest symptom would be degraded \emph{asymptotic} performance,
yet shaping raises solve reliability to 29/30 seeds against 21/30 for the no-KG control, and
under corruption it tracks that control rather than dropping meaningfully below it ($0.375$
pooled against $0.387$; $d{=}{-}0.04$, $p{=}0.87$). The mechanism's effects therefore
appear as changes in learning speed, not as a shifted optimum, though this is empirical evidence
on these tasks rather than a guarantee.

\section{Shuffle-Control Implementation}
Both variants operate on the retained triples $(h, r, t)$ after the completeness, noise, and
size knobs are applied, with the relation label $r$ bound to its slot, so relation \emph{type}
and edge \emph{direction} are preserved and only endpoints are rewired. Endpoint permutation
independently permutes the head and tail lists and re-pairs them with relations held fixed,
preserving the edge count and the head and tail multisets while destroying the head-to-tail
pairing. The degree-preserving variant instead performs $10|E|$ random tail-swaps between edge
pairs, leaving every head's out-degree and every tail's in-degree invariant yet still
randomizing which head connects to which tail. Both draw from a \texttt{random.Random(seed)}
instance seeded by the run's seed, so every (condition, seed) cell reproduces its shuffled graph
exactly.

\paragraph{Corruption process.}
The \texttt{corrupt\_critical} knob applies only to the $k$ retained critical triples. For each
such triple $(h,r,t)$ we keep the head and the relation label and replace the tail with an
entity drawn uniformly from the MiniGrid vocabulary (objects, colors, and states) excluding the
true tail, giving a well-formed but factually wrong relation such as
\textsf{(key, opens, green)}. Corruptions are drawn \emph{independently} per triple from the
same seeded generator, so the errors are neither correlated across relations nor structured at
the schema level. Real knowledge graphs plausibly fail in more clustered ways, for example
losing every relation of one type or every relation incident to one entity, and our corruption
model does not capture those; we therefore read the corruption results as characterizing
independent factual error rather than KG error in general.

\section{Reproducibility Statement}
All experiments are driven by declarative configuration files and a resumable grid runner, and
every run is seeded across Python, NumPy, and PyTorch. The synthetic KG generator, the three
injection wrappers, the agents, and the analysis code that produces every figure and table are
released. The clinical case study depends on the credentialed MIMIC-IV and UMLS releases,
which we do not redistribute; we release the cohort-construction, KG-extraction, and offline-RL
code together with the exact table and field identifiers used.

\end{document}